\documentclass{article}

\usepackage[preprint]{neurips_2026}

\usepackage[utf8]{inputenc} 
\usepackage[T1]{fontenc}    
\usepackage{hyperref}       
\usepackage{url}            
\usepackage{booktabs}       
\usepackage{amsfonts}       
\usepackage{nicefrac}       
\usepackage{microtype}      
\usepackage{xcolor}         
\usepackage{graphicx}       
\usepackage{amsmath}
\usepackage{float}
\usepackage{booktabs}
\usepackage{amsthm}

\usepackage{bm}  

\usepackage{listings}
\usepackage{xcolor}

\usepackage{caption}
\definecolor{codekw}{RGB}{0,0,180}      
\definecolor{codecom}{RGB}{0,128,0}     
\definecolor{codestr}{RGB}{163,21,21}   
\definecolor{codenum}{RGB}{128,128,128} 
\definecolor{codedef}{RGB}{111,0,138}   

\lstdefinestyle{mintedclean}{
  language=Python,
  basicstyle=\ttfamily\small,
  commentstyle=\color{codecom},
  keywordstyle=\color{codekw}\bfseries,
  stringstyle=\color{codestr},
  numberstyle=\tiny\color{codenum},
  identifierstyle=\color{black},
  emph={self,Tensor,nn,einsum,linear},
  emphstyle=\color{codedef},
  showstringspaces=false,
  breaklines=true,
  tabsize=4,
  keepspaces=true,
  columns=fullflexible,
  numbers=left,
  numbersep=6pt,
  upquote=true,
  frame=none,
  backgroundcolor=\color{white},
}

\newtheorem{theorem}{Theorem}
\theoremstyle{definition}
\newtheorem{definition}{Definition}

\DeclareMathOperator*{\argmax}{arg\,max}
\DeclareMathOperator{\Tr}{Tr}

\usepackage{xcolor}
\definecolor{orange}{HTML}{EC9940}
\definecolor{red}{HTML}{C35250}
\definecolor{green}{HTML}{98C24C}
\definecolor{purple}{HTML}{A573D3}
\definecolor{blue}{HTML}{78C3FB}

\usepackage{bm}

\newcommand{\codelink}{[see supplementary material]}
\renewcommand{\codelink}{\url{https://github.com/tdooms/bae}~}

\newcommand\wei[1]{\bm{\textcolor{red}{#1}}}
\newcommand\act[1]{\bm{\textcolor{blue}{#1}}}

\title{Bilinear autoencoders find interpretable manifolds}

\author{%
    Thomas Dooms\thanks{Equal contribution.}\\
    University of Antwerp\\
    \texttt{doomsthomas@gmail.com}
    \And
    Ward Gauderis\footnotemark[1]\\
    Vrije Universiteit Brussel\\
    \texttt{ward.gauderis@vub.be}
    \AND
    Geraint Wiggins\\
    Vrije Universiteit Brussel
    \And
    Jose Oramas\\
    University of Antwerp, sqIRL
}

\begin{document}

\maketitle

\begin{abstract}
Sparse autoencoders have become a standard tool for uncovering interpretable latent representations in neural networks.
Yet salient concepts often span manifolds that current linear methods cannot capture without post hoc analysis.
This paper uses quadratic latents to close this gap: we implement these with bilinear autoencoders, which decompose activations into low-rank quadratic forms, compose linearly in weight space, and admit input-independent geometric analysis.
This qualitative difference in what concepts quadratic latents can detect challenges the standard linear representation hypothesis.
Our experiments and visualisations show that multi-dimensional geometries are highly prevalent and that composite latents capture them well, systematically improving reconstruction error in language models.
Furthermore, we show that autoencoders with varying geometric priors recover the same input subspace despite their dictionary entries being distinct.
Practically, these models serve as an unsupervised tool for manifold discovery, which we demonstrate through an interactive online visualizer for Qwen 3.5.
This is a step toward nonlinear but mathematically tractable latent representations whose composition is expressive and interpretable by design.
\end{abstract}

\begin{figure}[h]
  \centering
  \includegraphics[width=0.9\textwidth]{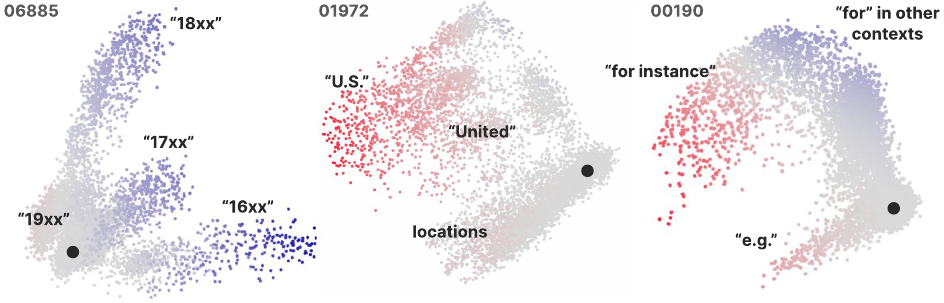}
  \caption{Three hand-picked latent features from our interactive viewer using Qwen 3.5: \url{https://bae-9xf.pages.dev/}. Our proposed bilinear autoencoders capture low-dimensional manifolds. The above visuals are created by linearly projecting activations of input samples into the (3D) eigenspace of three different dictionary elements. The left latent fires selectively on years whose geometry represents them as superposed linear subspaces. The middle latent activates on U.S.-related tokens, and the inactive (grey) samples correspond to other locations. The geometry is scattered across all 3 dimensions, containing visible semantic clusters. The right latent activates positively on `for instance' and `for example' and negatively on `for' in other contexts. Refer to \autoref{sub:geometry} for further explanation and to \autoref{app:examples} for random examples.}
  \label{fig:manifolds}
\end{figure}

\section{Introduction}
\label{sec:intro}
Neural networks are highly nonlinear systems shaped by emergent phenomena rather than explicit mechanisms. Logical human understanding, by contrast, is grounded in explicit structure and predictable composition. Mechanistic interpretability attempts to bridge this gap through reverse-engineering: identifying interpretable bases \citep{olah2022mechanistic,cunningham2023sparseautoencodershighlyinterpretable} and the circuits that act upon them. This lets us examine latent features in isolation, despite their co-dependence \citep{ameisen2025circuit}. Yet the core problem remains: these latents often lack a compositional interpretation, since their isolated behaviour does not determine their role in the whole \citep{gauderis_mechanistic_2026, coecke2021compositionalityit}. Extracted latents have implicit domains of applicability, regions in input or activation space within which they are meaningful. Latent splitting and absorption \citep{chanin2025absorptionstudyingfeaturesplitting} illustrate this: a latent's meaning is not fixed but depends on which other latents are active alongside it, so isolated study leaves this context-dependence invisible \citep{bhalla2026sparseautoencoderscaptureconcept}. The isolated study of latents is therefore of limited value, since their interaction may skew interpretation even at small scales \citep{méloux2025everythingeverywhereoncemechanistic}. Compositional understanding requires interactions that are explicit and computable \citep{tullCompositionalInterpretabilityXAI2024, gauderis_mechanistic_2026}.

Polynomials are a feature class with explicit compositional properties. They are nonlinear
in input, but the parameters that define them combine linearly: a sum of quadratic latents
is itself quadratic, parameterised by the sum of the original parameter vectors
\citep{doomsCompositionalityUnlocksDeep24, pearce2025bilinearmlpsenableweightbased}.
Linear composition therefore happens in weight space, not on activations. We add,
decompose, or compare latents by manipulating their coefficients directly, with no forward
pass and no data. Their geometry is read off these coefficients and spectral
properties \citep{lewisCompositionalHyponymyPositive2019}. Such a latent can carve a surface, isolate a cluster, or trace an ellipsoid
through a language model's activation space (\autoref{fig:manifolds}). Recent work
documents these structures through input-dependent and nonlinear projections
\citep{engels2025languagemodelfeaturesonedimensionally, modell2025originsrepresentationmanifoldslarge},
which recover the empirical geometry but sever its tie to the parameters that produced it.

This paper develops bilinear autoencoders as a lens on representation geometry. They decompose activations into a dictionary of quadratic latents whose receptive fields admit closed-form geometric analysis. We motivate the architecture with a framework of structural priors that map directly to manifold shapes, and validate them against modern language models. We make three contributions\footnote{The code is publicly available at \codelink.}
\begin{itemize}
\item \textbf{Architecture.} We adapt bilinear autoencoders to dictionary learning, with efficient training via a kernel trick, and motivate the architecture as a natural and tractable extension of existing linear sparse autoencoders (\autoref{sec:architecture}).
\item \textbf{Framework.} We propose three nested representation hypotheses (atomic, composite, quadratic), formalising how structural priors on the dictionary translate into receptive-field geometries: slabs, ellipsoids and hyperboloids, and arbitrary quadrics (\autoref{sec:structure}).
\item \textbf{Empirics.} Multi-dimensional geometries are prevalent in language models and composite latents capture them with lower reconstruction error. Bilinear autoencoders recover stable input subspaces across priors despite learning disjoint dictionaries (\autoref{sec:empirics}).
\end{itemize}

We write scalars (0-order tensors) as $s$ (lower-case), vectors (1-order tensors) as $\mathbf{v}$ (bold lower-case), matrices (2-order tensors) as $M$ (upper-case) and higher-order tensors as $\mathbf{T}$ (bold upper-case). Indexed tensors are denoted as $\mathbf{T}_{ijk}$ and slicing along tensor dimensions as $\mathbf{T}_{m:n}$. We use colour to distinguish \wei{weight-based} from \act{activation-based} quantities; a glossary is provided in \autoref{app:notation}.

\section{Bilinear autoencoders}
\label{sec:architecture}

Autoencoders reconstruct their inputs through a bottleneck, forcing the model to capture structure in the data \citep{hinton2006reducingdimensionality, hindupur2025projectingassumptionsdualitysparse}. They have become a key tool in interpretability for extracting meaningful latents from neural representations \citep{olshausen1996emergence, cunningham2023sparseautoencodershighlyinterpretable, bricken2023monosemanticity}. Their interpretative goal is dual: reconstruct the input, and trace the extracted latents back to the original model. These goals pull against each other, trading expressivity for analysability.

The truncated singular value decomposition is itself a dimensionality-reducing autoencoder, but its compression is too restrictive: neural representations typically span all dimensions, so rank-based methods miss much of their structure \citep{elhage2022toymodelssuperposition}. Sparse autoencoders instead focus on sparsity in a high-dimensional latent space, adding a sparsity-inducing nonlinearity to the encoder while keeping the decoder linear \citep{lewicki2000overcomplete, aharon2006ksvd, mairal2009onlinedictionary}. They recover linearly separable latents that trace back to the original representations. 
Because each approach assumes a fixed geometry, the resulting latents mirror the imposed prior instead of the underlying structure of the data.
\citep{engels2025languagemodelfeaturesonedimensionally, buchananLearningDeepRepresentations, hindupur2025projectingassumptionsdualitysparse}.

\subsection{Geometric priors for dictionary learning} \label{sub:priors}
Current autoencoder architectures rest on the \textit{linear representation hypothesis}: the claim that neural representations encode concepts mostly along linear directions \citep{park2024linearrepresentationhypothesisgeometry, cunningham2023sparseautoencodershighlyinterpretable, gao2024scalingevaluatingsparseautoencoders, rajamanoharan2024jumpingaheadimprovingreconstruction}. Yet, recent research shows that salient concepts often span more intricate manifolds that existing methods cannot capture \citep{modell2025originsrepresentationmanifoldslarge, engels2025languagemodelfeaturesonedimensionally, bhalla2026sparseautoencoderscaptureconcept}. Increasing capacity is a natural response, but unbounded expressivity collapses to trivial solutions and defeats interpretation \citep{gao2024scalingevaluatingsparseautoencoders, sutter2025nonlinearrepresentationdilemmacausal}.

Using autoencoders for interpretability is inherently about hypotheses \citep{hindupur2025projectingassumptionsdualitysparse} and finding the lens under which representations become understandable. Yet we do not know the neural representation geometry in advance; autoencoder architecture and its training should guide discovery toward simple latents without prescribing their form. We need primitives that are simple individually but compose into latents of varying complexity \citep{jenatton2011hierarchical, bhalla2026sparseautoencoderscaptureconcept}. This requires measuring the complexity of the learned feature class itself, not just of the activations it produces. Minimum description length \citep{gauderis_mechanistic_2026, rissanen1978mdl, grunwald2007minimum, ayonrinde2024interpretabilitycompressionreconsideringsae} is the natural candidate, but it requires a model class in which architectural description length independent of data is well-defined. Polynomials fit this role: their closed-form composition produces increasingly intricate, yet well-understood manifolds in their input space.

\subsection{From gated units to quadratic latents} \label{sub:architecture}

The simplest nonlinear instance of this idea is to model latents as quadratics. We implement
this using bilinear autoencoders, which adapt
sparse autoencoders to gated linear units \citep{shazeer2020gluvariantsimprovetransformer,
dauphin2017languagemodelinggatedconvolutional}. These units apply two matrices
$\wei{L}, \wei{R} \in \mathbb{R}^{h \times d}$ to their input $\act{\mathbf{x}} \in \mathbb{R}^d$
and project their element-wise multiplication using $\wei{C} \in \mathbb{R}^{k \times h}$.
\[
\mathrm{GLU}(\act{\mathbf{x}}) \;=\; \wei{C}\big(\sigma(\wei{L}\act{\mathbf{x}}) \odot \wei{R}\act{\mathbf{x}}\big)
\]
Removing the gating activation $\sigma$ has a negligible impact on accuracy
\citep{pearce2025bilinearmlpsenableweightbased, shazeer2020gluvariantsimprovetransformer}.
What remains computes second-order polynomials of $\act{\mathbf{x}}$, but remains linear in
the `lifted' \textit{product space} of pairwise interactions
$\act{X} = \act{\mathbf{x}}\act{\mathbf{x}}^\top$. The $i$-th latent activation is
\[
\act{z}_i \;:=\; \act{\mathbf{x}}^\top \wei{L}^\top \operatorname{diag}(\wei{\mathbf{c}}_i)\, \wei{R}\, \act{\mathbf{x}} \;=\; \langle \wei{W}_i,\, \act{X} \rangle_F
\]
where $\wei{\mathbf{c}}_i$ is the $i$-th row of $\wei{C}$ and $\wei{W}_i$ is the
bilinear form of the $i$-th latent,
\[
\wei{W}_i \;=\; \sum_{j \in \mathrm{supp}(\wei{\mathbf{c}}_i)} \wei{C}_{ij}\, \wei{\mathbf{l}}_j\wei{\mathbf{r}}_j^\top.
\]
Stacking all $\wei{W}_i$ into a third-order tensor gives the full encoder
$\wei{\mathcal{W}} \in \mathbb{R}^{k \times (d \times d)}$.
The autoencoder's training objective is to approximate $\act{X}$ through a symmetric
bottleneck:
\[
\act{\hat{X}} \;=\; \wei{\mathcal{W}}^\top \wei{\mathcal{W}}\, \act{X}.
\]
Lifting the input to $\act{X} = \act{\mathbf{x}}\act{\mathbf{x}}^\top$ turns the
autoencoding problem into a linear factorisation of quadratic latents. This representation
makes the geometric structure explicit: quadric activation regions, such as ellipsoids,
paraboloids and hyperboloids, are level sets of linear functionals of $\act{X}$
(\autoref{sec:structure}; \autoref{fig:geometry}). This lifted space has $\Theta(d^2)$
dimensions, making direct materialisation of $\act{X}$ intractable at scale; we instead
evaluate the loss via a kernel trick (\autoref{app:efficiency}). Since the model cannot
degenerate to identity-like maps---as ordinary SAEs do without sparsity penalties---its
bottleneck lies in which combinations to preserve. In other words, the learned
$\wei{\mathcal{W}}$ specifies exactly which second-order structure is kept and discarded.

\begin{figure}[H]
  \centering
  \includegraphics[width=0.7\textwidth]{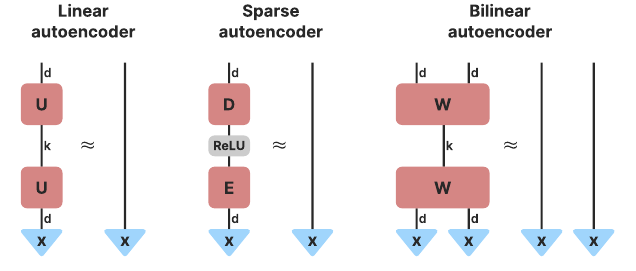}
  \caption{Most autoencoders reconstruct their inputs nonlinearly. Instead, bilinear
  autoencoders linearly reconstruct $\act{X} = \act{\mathbf{x}}\act{\mathbf{x}}^\top \cong \act{\mathbf{x}}\otimes\act{\mathbf{x}}$ in the $d \times d$ product space, which is quadratic in the input $\act{\mathbf{x}}$.}
  \label{fig:autoencoders}
\end{figure}

\subsection{Scale-invariant activation sparsity} \label{sub:hoyer}

The choice of $\wei{\mathcal{W}}$ governs which second-order patterns survive; activation sparsity governs how cleanly each fires. The two are complementary. The standard approach minimises sparsity-inducing norms ($L_{0 < p \le 1}$) on activations, but causes trivial shrinkage and can destabilise training by producing dead latents \citep{bricken2023monosemanticity}. Both pull against our goal: we want \emph{relative} prominence across the sample dimension $\act{\mathbf{z}}_i \in \mathbb{R}^n$ (effective batch size $n = \text{samples} \times \text{sequence length}$), not uniformly small activations \citep{hoyer}.
\begin{equation}
\mathrm{Hoyer}_{\text{density}}(\act{\mathbf{z}}_i - \wei{b}_i) \;=\; \frac{\|\act{\mathbf{z}}_i - \wei{b}_i\|_1/\|\act{\mathbf{z}}_i - \wei{b}_i\|_2 - 1}{\sqrt{n}-1}
\label{eq:hoyer}
\end{equation}
This measure runs from 0 for one-hot vectors to 1 for perfectly uniform ones, and is invariant to scale by construction and, through $\wei{b}_i$, to translation: each latent learns the offset that makes its activation distribution as sparse as possible. The penalty is minimised when activations concentrate around any level set $\act{\mathbf{z}}_i = \wei{b}_i$, rewarding latents whose level sets align with structure in the data. It pressures latents to fire selectively on inputs without collapsing to zero \citep{bussmann2024batchtopksparseautoencoders}.


\section{Structural constraints of bilinear autoencoders}
\label{sec:structure}

Any sparse autoencoder enforces two kinds of structural constraints: the decoder operationalises a hypothesis about how neural representations are constructed from concepts \citep{felRabbitHullTaskRelevant2025}, while the encoder's receptive fields encode assumptions about how those concepts appear in the activation space \citep{hindupur2025projectingassumptionsdualitysparse}. In bilinear autoencoders, encoder and decoder are mirrors of each other; aligned both sides by construction. This section unpacks what each side implies, working through three nested hypotheses of increasing structure: the quadratic hypothesis (\autoref{sub:quadratic}), the atomic refinement (\autoref{sub:atoms}), and the composite refinement (\autoref{sub:composites}).

\subsection{The quadratic representation hypothesis}
\label{sub:quadratic}

Sparse autoencoders operationalise the \emph{linear representation hypothesis}
\citep{elhage2022toymodelssuperposition,felRabbitHullTaskRelevant2025}: representations
decompose into a sparse sum of directions.
\begin{definition}[Linear Representation Hypothesis]
A representation $\act{\mathbf{x}} \in \mathbb{R}^d$ satisfies the \emph{linear
representation hypothesis} if there exists a dictionary of directions
$\wei{\mathbf{w}}_i \in \mathbb{R}^d$ and a sparse activation vector
$\act{\mathbf{z}} \in \mathbb{R}^k$ such that
$\act{\mathbf{x}} = \sum_{i \in \mathrm{supp}(\act{\mathbf{z}})} \act{z}_i\, \wei{\mathbf{w}}_i$.
\end{definition}
Bilinear decoders instead reconstruct $\act{X} = \act{\mathbf{x}}\act{\mathbf{x}}^\top$,
which lifts the problem to linear dictionary learning on the \emph{Veronese variety}
$\mathcal{V} = \{\act{\mathbf{x}}\act{\mathbf{x}}^\top : \act{\mathbf{x}} \in \mathbb{R}^d\}
\subset \mathbb{R}^{d\times d}_{\mathrm{sym}}$. This is the simplest non-linear relaxation
of the linear hypothesis, and we formalise it accordingly.
\begin{definition}[Quadratic Representation Hypothesis]
A representation $\act{\mathbf{x}} \in \mathbb{R}^d$ satisfies the \emph{quadratic
representation hypothesis} if there exists a dictionary of symmetric matrices
$\wei{W}_i \in \mathbb{R}^{d\times d}_{\mathrm{sym}}$ and a sparse activation vector
$\act{\mathbf{z}} \in \mathbb{R}^k$ such that
$\act{\mathbf{x}}\act{\mathbf{x}}^\top = \sum_{i \in \mathrm{supp}(\act{\mathbf{z}})} \act{z}_i\, \wei{W}_i$.
\end{definition}
We adopt these hypotheses not as empirical claims but as the operationalising assumptions
required for any form of dictionary learning
\citep{nguyenKernelDictionaryLearning2012,felRabbitHullTaskRelevant2025}; whether neural
representations actually satisfy the quadratic hypothesis is an empirical question which we
address in \autoref{sec:empirics}. Under it, the dictionary atoms $\wei{W}_i$ are
symmetric matrices rather than directions: each encodes a specific pattern of features
that co-vary together, with diagonal entries $\act{x}_i^2$ capturing the energy of
individual dimensions and off-diagonal entries $\act{x}_i \act{x}_j$ encoding pairwise
co-activations. Reconstructing $\act{X}$ sparsely is therefore equivalent to decomposing
the sample's covariance structure into a small set of fundamental interaction patterns,
equivalent to reconstruction ICA on the Veronese embedding \citep{leICAReconstructionCost}.
Where an SAE asks \emph{which directions activate}, a bilinear autoencoder asks
\emph{which co-occurrence patterns are active}.

Two structural consequences follow directly from targeting $\act{X}$ rather than
$\act{\mathbf{x}}$. First, \emph{second-order sufficiency}: a feature that survives
training must explain not merely which dimensions activate but how they activate together,
a strictly harder constraint that makes signed and mean-field structure inaccessible to
the decoder. Second, \emph{antipodal symmetry}: since
$(-\act{\mathbf{x}})(-\act{\mathbf{x}})^\top = \act{\mathbf{x}}\act{\mathbf{x}}^\top$,
the model inherently ignores global phase, making it the natural choice for features
stored in antipodal superposition \citep{elhage2022toymodelssuperposition} where a signed
direction $\pm\wei{\mathbf{w}}$ encodes a single concept and a standard SAE must spend
two features on what is geometrically one direction.

\subsection{Rank-1 atoms and energy separability}
\label{sub:atoms}

Just as SAEs build their dictionary from vectors, the simplest linear objects, bilinear
autoencoders are most naturally studied when their atoms are rank-1 matrices, the simplest
quadratic objects.



\begin{definition}[Atomic Representation Hypothesis]
A representation satisfies the \emph{atomic representation hypothesis} if the dictionary
atoms in the quadratic representation hypothesis are rank-1:
$\wei{W}_i = \wei{\mathbf{w}}_i\wei{\mathbf{w}}_i^\top$.
\end{definition}
Reconstructing $\act{\mathbf{x}}\act{\mathbf{x}}^\top = \sum_{i \in \mathrm{supp}(\act{\mathbf{z}})} \act{z}_i\, \wei{\mathbf{w}}_i\wei{\mathbf{w}}_i^\top$ assumes that the covariance structure decomposes into a sparse sum of rank-1 prototypes; a few directions alone jointly account for the observed interactions.

\begin{figure}[t]
  \centering
  \includegraphics[width=0.71\textwidth]{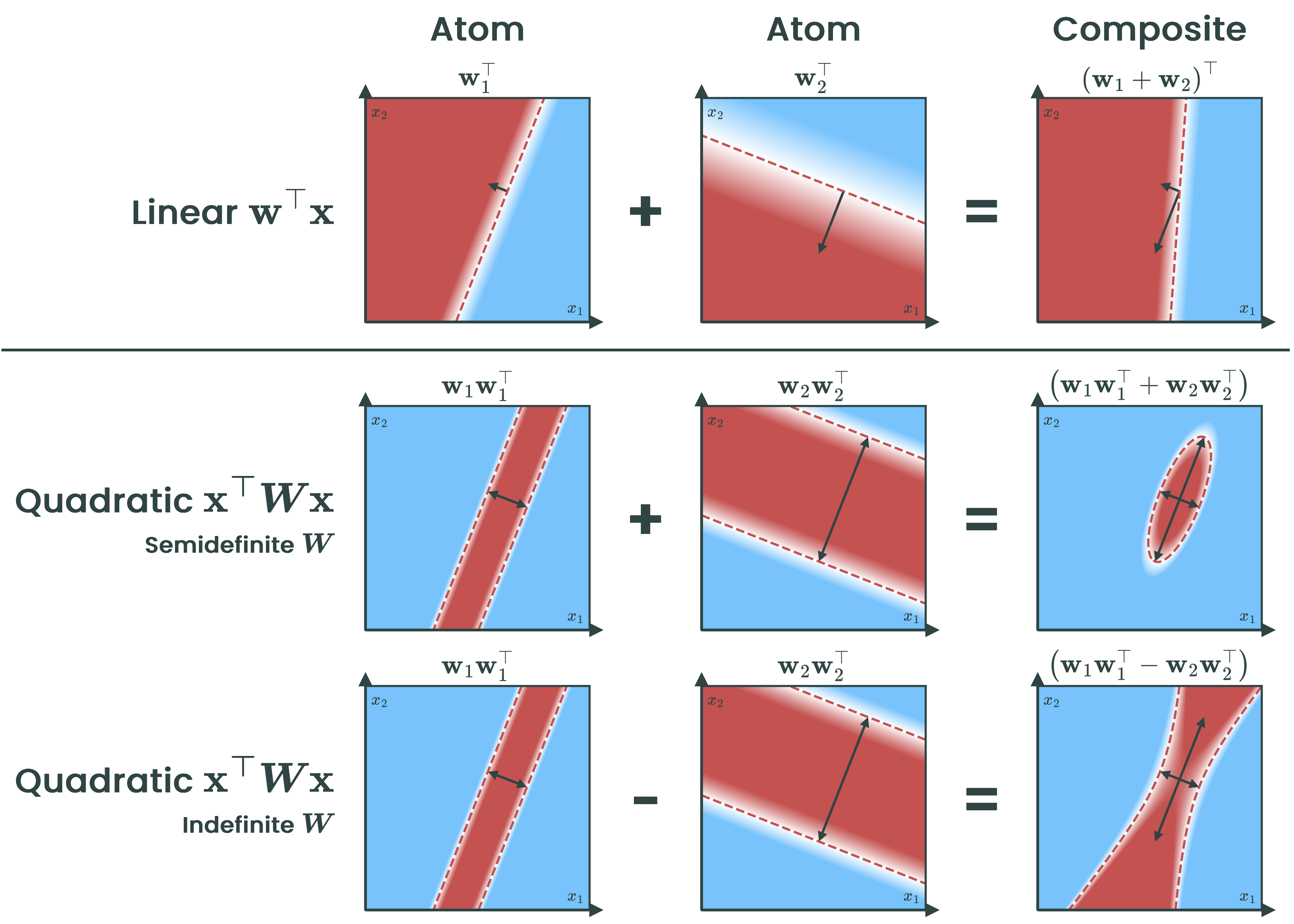}
    \caption{%
    \textbf{Linear atoms gain no expressivity under composition; quadratic atoms compose into quadric geometry.}
    (\textit{Top}) Linear atoms are signed half-space detectors: their sum is another half-space, merely averaging the directions.
    (\textit{Middle and bottom}) Quadratic atoms are symmetric slabs measuring energy along a direction, ignoring phase. Composing two rank-1 forms yields a rank-2 symmetric matrix, unlocking quadric geometries. The sum is a semidefinite \emph{ellipsoid}, a conjunction detector for $\wei{\mathbf{w}}_1$ \emph{\color{red} and} $\wei{\mathbf{w}}_2$; the difference is an indefinite \emph{hyperboloid}, a detector for $\wei{\mathbf{w}}_1$ \emph{\color{red} and not} $\wei{\mathbf{w}}_2$.
    Linear composites span a flat set of directions; quadratic composites generate an algebra of symmetric matrices whose geometry ranges over paraboloids, ellipsoids, and hyperboloids.%
  }
  \vspace*{-1em}
  \label{fig:geometry}
\end{figure}

\paragraph{Receptive fields.}
The encoder traces this hypothesis back to the input space, and the difference from a
standard SAE determines what concept geometry each can detect. An SAE computes a signed
projection and thresholds it, while an atomic bilinear encoder computes a squared
projection:
\begin{equation}
  \act{z}_i^{\mathrm{lin}} = \mathrm{ReLU}(\wei{\mathbf{w}}_i^\top \act{\mathbf{x}}),
  \qquad
  \act{z}_i^{\mathrm{atom}} = (\wei{\mathbf{w}}_i^\top \act{\mathbf{x}})^2
  = \langle \wei{\mathbf{w}}_i\wei{\mathbf{w}}_i^\top,\, \act{X} \rangle_F.
\end{equation}

The second equality shows that each bilinear measurement is a linear functional on
$\act{X}$, the inner product of the Veronese image with the rank-1 atom
$\wei{\mathbf{w}}_i\wei{\mathbf{w}}_i^\top$, or equivalently a Born-rule measurement:
the expectation value of the projector $\wei{\mathbf{w}}_i\wei{\mathbf{w}}_i^\top$ on
the pure state $\act{X}$. The encoder is a projection encoder
\citep{hindupur2025projectingassumptionsdualitysparse} from the Veronese variety onto
the space of diagonal matrices.

The two receptive fields, for some threshold $\tau \ge 0$, are different:
\begin{equation}
  \mathcal{F}_i^{\mathrm{lin}} = \{\act{\mathbf{x}} : \wei{\mathbf{w}}_i^\top\act{\mathbf{x}} > \tau\},
  \qquad
  \mathcal{F}_i^{\mathrm{atom}} = \{\act{\mathbf{x}} : (\wei{\mathbf{w}}_i^\top\act{\mathbf{x}})^2 > \tau\}.
\end{equation}
The SAE's is a directed half-space \citep{hindupur2025projectingassumptionsdualitysparse}; the bilinear's is a \emph{symmetric slab}, two
parallel hyperplanes around $\wei{\mathbf{w}}_i^\perp$ and symmetric about the origin.
Where a linear atom asks \emph{which side of a hyperplane $\act{\mathbf{x}}$ lies on},
a quadratic atom asks \emph{how close to a hyperplane $\act{\mathbf{x}}$ is}, regardless
of sign. The respective data assumptions are \emph{linear separability} and \emph{energy
separability}: a concept detectable from the sign and magnitude of a projection, versus
from its magnitude alone. See \autoref{fig:geometry} for a visual comparison\footnote{Although $\wei{\mathbf{l}}_i, \wei{\mathbf{r}}_i$
need not be equal in practice (\autoref{sec:architecture}), this is without loss of generality: since
$\act{\mathbf{x}}^\top A \act{\mathbf{x}} = \act{\mathbf{x}}^\top \tfrac{A + A^\top}{2}
\act{\mathbf{x}}$ for any $A$, only the symmetric part
$\wei{W}_i = \tfrac{1}{2}(\wei{\mathbf{l}}_i\wei{\mathbf{r}}_i^\top +
\wei{\mathbf{r}}_i\wei{\mathbf{l}}_i^\top)$ contributes to $\act{z}_i$. This has rank at most 2 and its eigendecomposition reduces
to a signed sum of two rank-1 atoms, i.e.\ the composite case of
\autoref{sub:composites}.}.

\paragraph{Intrinsic sparsity.}
The slab geometry has a further consequence: quadratic atoms are far more selective than
linear ones in high dimensions. The proof is in \autoref{app:proof}.

\begin{theorem}[Exponential receptive field gap]
\label{thm:receptive-field}
Let $\wei{\mathbf{w}} \in S^{d-1}$ and $\act{\mathbf{x}} \sim \mathrm{Unif}(S^{d-1})$.
For any fixed threshold $\tau \in (0,1)$,
\[
  \frac{\mathbb{P}\bigl((\wei{\mathbf{w}}^\top\act{\mathbf{x}})^2 > \tau\bigr)}
       {\mathbb{P}\bigl(\wei{\mathbf{w}}^\top\act{\mathbf{x}} > \tau\bigr)}
  \;=\; \Theta\!\left(e^{-d\tau(1-\tau)/2}\right).
\]
\end{theorem}
Quadratic atoms achieve intrinsic activation sparsity: their receptive fields shrink
exponentially faster with dimension than those of linear atoms. At $\tau = 1/2$ and
$d = 1024$, the ratio is $\approx 10^{-56}$.

\subsection{Composite atoms and quadric separability}
\label{sub:composites}

Linear atoms compose poorly: summing two SAE atoms yields another half-space, merely
reoriented (\autoref{fig:geometry}). \citet{bhalla2026sparseautoencoderscaptureconcept}
identify this as the reason SAEs tile manifolds through many fragmented atoms rather than
capturing them with a compact group, and call for featurizers that treat geometric objects
as primitives from the outset. Quadratic atoms answer this directly: composing two rank-1
forms produces a rank-2 symmetric matrix whose receptive field is an ellipsoid or
hyperboloid rather than a slab (\autoref{fig:geometry}). We formalise the hypothesis that
concepts exploit this structure.

\begin{definition}[Composite Representation Hypothesis]
A representation satisfies the \emph{composite representation hypothesis} if the
dictionary atoms in the quadratic representation hypothesis are sparse mixtures of
rank-1 matrices over a shared base dictionary
$\{\wei{\mathbf{w}}_j\}_{j=1}^h \subset \mathbb{R}^d$:
$\wei{W}_i = \sum_{j \in \mathrm{supp}(\wei{\mathbf{c}}_i)} \wei{C}_{ij}\,
\wei{\mathbf{w}}_j\wei{\mathbf{w}}_j^\top$,
where $\wei{\mathbf{c}}_i$ is sparse and fixed per composite $\wei{W}_i$.
\end{definition}

Substituting into the quadratic hypothesis gives
$\act{\mathbf{x}}\act{\mathbf{x}}^\top = \sum_{i \in \mathrm{supp}(\act{\mathbf{z}})} \act{z}_i \sum_{j \in \mathrm{supp}(\wei{\mathbf{c}}_i)} \wei{C}_{ij}\,
\wei{\mathbf{w}}_j\wei{\mathbf{w}}_j^\top$: the representation is reconstructed through two levels of sparsity, a dynamic activation $\act{\mathbf{z}}$ over composites and a static grouping $\wei{\mathbf{c}}_i$ of base atoms. This is related to hierarchical dictionary methods \citep{luoAtomsTreesBuilding2026}, but differs in that the grouping arises from a single global sparsity penalty on $\wei{C}$ rather than an added loss term, letting the optimiser allocate capacity unevenly across composites.

\paragraph{Receptive fields.}
A composite encoder computes a weighted sum of squared projections, equivalently the
expectation value of a low-rank mixture of atomic projectors:
\begin{equation}
  \act{z}_i^{\mathrm{comp}} \;=\; \act{\mathbf{x}}^\top \wei{W}_i \act{\mathbf{x}}
  \;=\; \sum_j \wei{\lambda}_j^{(i)} \bigl(\wei{\mathbf{v}}_j^{(i)\top}\act{\mathbf{x}}\bigr)^2,
  \qquad
  \wei{W}_i \;=\; \sum_{j \in \mathrm{supp}(\wei{\mathbf{c}}_i)} \wei{C}_{ij}\,
  \wei{\mathbf{w}}_j\wei{\mathbf{w}}_j^\top,
\end{equation}
where $\wei{\lambda}_j^{(i)}, \wei{\mathbf{v}}_j^{(i)}$ are the eigenvalues and
eigenvectors of $\wei{W}_i$. When $\wei{W}_i$ is semidefinite, it aggregates slab
detectors into a single quadratic form whose receptive field is an \emph{ellipsoidal
region}: an ellipsoid with principal axes given by the eigenvectors of $\wei{W}_i$ and
radii set by the eigenvalues. When $\wei{W}_i$ is indefinite, it combines slab detectors
with opposite signs, and the receptive field becomes a \emph{hyperboloidal region}: a
hyperboloid whose positive sheet aligns with the positive eigenspace and the negative
sheet with the negative eigenspace. In both cases, a composite atom asks \emph{how close
to a quadric surface $\act{\mathbf{x}}$ is}, a strict generalisation of the atomic
question of how close to a hyperplane. The data assumption is \emph{low-rank quadratic
separability}: a concept is detectable from the projection of $\act{\mathbf{x}}$ onto a
small subspace, with the subspace structure encoded in the eigenvectors of $\wei{W}_i$
(\autoref{fig:geometry}).

A rank-$r$ composite can achieve compact subspace capture
\citep{bhalla2026sparseautoencoderscaptureconcept} of any manifold living in the
$r$-dimensional span of its eigenvectors, something that requires $r$ separate atoms in a
linear SAE: a circle needs a rank-2 composite, a sphere a rank-3 one. Whereas
\citet{bhalla2026sparseautoencoderscaptureconcept} recover such groupings post-hoc from
co-activation statistics, the sparse mixing layer $\wei{C}$ discovers them during
training, baking manifold subspaces directly into the dictionary.

\section{Empirical analysis of quadratic latents} \label{sec:empirics}

The three hypotheses progressively weaken the constraint on $\wei{C}$, increasing expressivity and decreasing interpretability. \emph{Atomic} fixes $\wei{C} = I$, making each latent a single rank-1 atom. \emph{Composite} imposes a global top-$K$ on $\wei{C}$ (allowing $0.1\%$ non-zero entries), letting the optimiser allocate atoms unevenly across composites. \emph{Quadratic} lifts the constraint entirely, freeing latents to reach arbitrary rank. Instantiating the full third-order tensor is impractical at scale, so it is approximated with a rank-$h$ polyadic decomposition \citep{polyadicdecomp}. See \autoref{app:setup} for the experimental setup.

The question is where this trade-off sits in practice. We measure reconstruction across priors, the prevalence of multi-dimensional geometry under each, and the stability of the recovered structure across runs, then qualitatively assess the extracted geometries through our interactive viewer.

\subsection{Composite latents improve reconstruction} \label{sub:reconstruction}

The simplest way to verify the utility of composites is to measure their reconstruction quality. Across even layers of Qwen 3.5 ($0.8 \text{B}$), the three regimes order cleanly: quadratic reconstructs best, composite next, atomic worst (\autoref{fig:layers}). Atomic latents are too rigid; composition recovers structure they miss, and the gap between composite and quadratic suggests that even low-rank composites fall short. This provides systematic evidence that semantic manifolds inhabit far more ambient directions than their intrinsic dimension requires \citep{olah2024manifolds}.

\begin{figure}[H]
  \centering
  \includegraphics[width=0.9\textwidth]{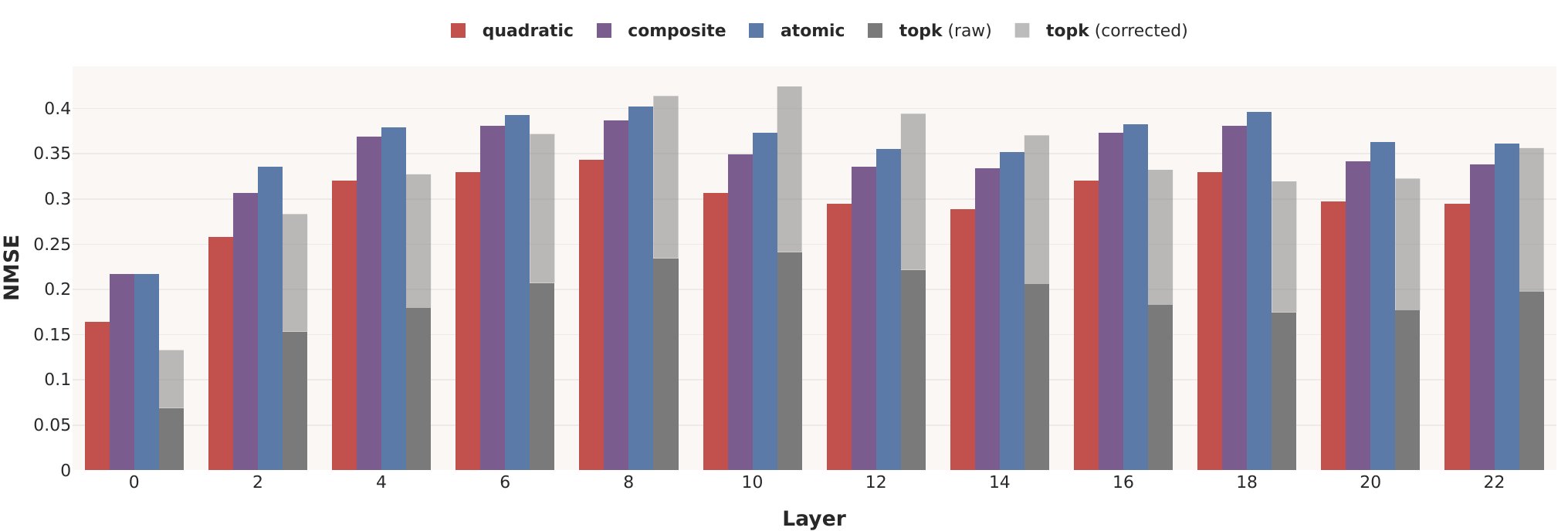}
  \caption{Reconstruction error across even layers. The three priors consistently follow the same error ordering, $\texttt{quadratic} < \texttt{composite} < \texttt{atomic}$, indicating that composite latents capture more structure than their single-rank counterparts. We use a TopK ($K=32$) baseline \citep{gao2024scalingevaluatingsparseautoencoders}, whose reconstruction error is corrected to account for the lifted space (see \autoref{app:correction}).}
  \label{fig:layers}
\end{figure}

\subsection{Multi-dimensional geometries are prevalent in language models} \label{sub:geometry}

How prevalent are multi-dimensional geometries and what do they look like \citep{modell2025originsrepresentationmanifoldslarge, engels2025languagemodelfeaturesonedimensionally, bhalla2026sparseautoencoderscaptureconcept}? For this, we consider the spectral properties of the latents from the autoencoder on layer 18 of Qwen 3.5. Results depend on the structural regime: most composite latents are well-captured by 5 dimensions while the quadratic latents are consistently high-rank, as seen in \autoref{fig:composites}. The former indicates that, while much structure is single-dimensional, there is a long tail of more complicated geometries that are commonly missed \citep{sharkeyOpenProblemsMechanistic2025}. The latter indicates that this structure does not emerge unless explicitly required.

\begin{figure}[H]
  \centering
  \includegraphics[width=0.9\textwidth]{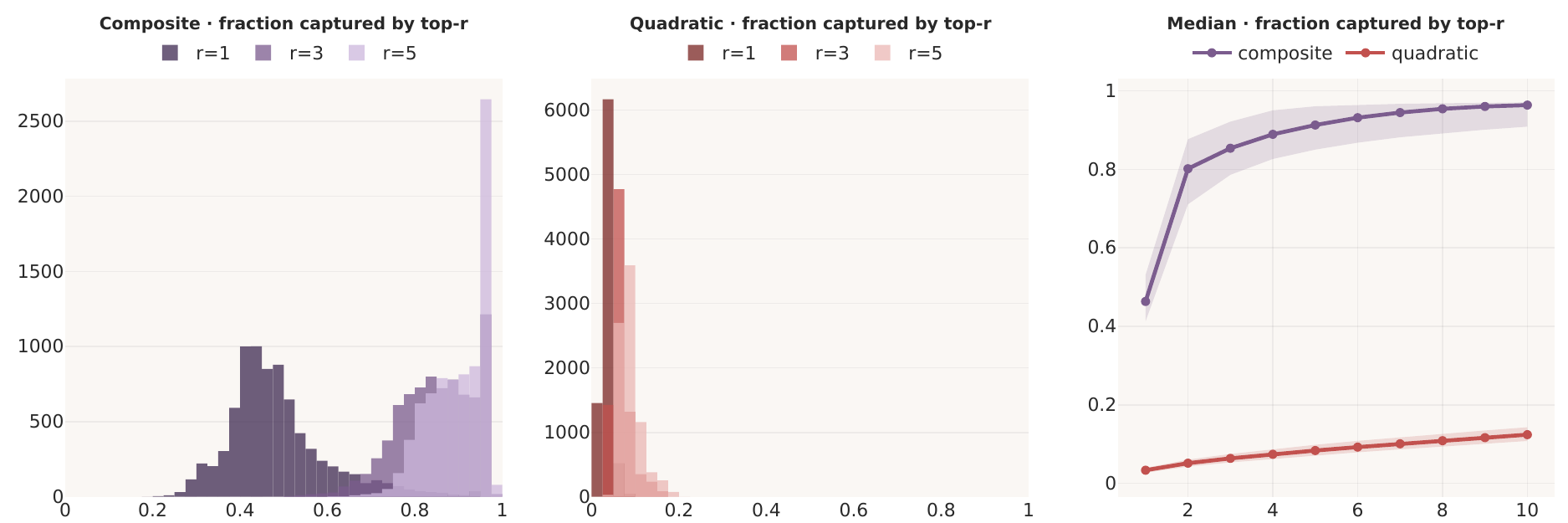}
  \caption{Comparison of the captured structure across ranks between quadratic and composite priors. The histograms show the captured variance of all latents for the top 1, 3 and 5 dimensions while the right plot shows their median as rank increases. With sparsity applied, most latents are very low rank with a tail of higher-dimensional structure. Without sparsity, the majority of latents are full-rank.}
  \label{fig:composites}
\end{figure}

Next, we turn toward a qualitative assessment of our \href{https://bae-9xf.pages.dev/}{interactive viewer} from which we discuss four common geometric patterns. The observed geometry is much richer than a linear prior would predict; especially the affine nature of certain concepts is interesting (e.g., two right manifolds in \autoref{fig:manifolds}).

\textbf{Cone:} a concept spans a line or cone, the simplest geometry. \href{https://bae-9xf.pages.dev/composite/06885}{6885} activates on specific years; \href{https://bae-9xf.pages.dev/composite/07552}{7552} activates on instances of the token `c'. The 3D visualisation often shows these latents in superposition, with non-orthogonal directions sharing a latent. \\
\textbf{Cluster:} A subspace populated by discrete groupings without linear structure. \href{https://bae-9xf.pages.dev/composite/01972}{1972} activates on references to the USA, with small clusters around each of the tokens in `U.S.'. \href{https://bae-9xf.pages.dev/composite/00152}{152} separates instances of ``said'' by context. Linear atoms could tile the point cloud but would lose the grouping. \\
\textbf{Interpolation:} A meaningful concept that smoothly interpolates between others. \href{https://bae-9xf.pages.dev/composite/02753}{2753} puts instances of `drug' between common addictions (coffee, smoking, alcohol). \href{https://bae-9xf.pages.dev/composite/00384}{384} activates on `de' in names (Catherine `de' Medici), spanning between the English `of' (Joan `of' Arc) and unrelated `de' tokens. \\
\textbf{Subset:} A latent that activates with opposite signs on distinct semantic regions of similar concepts. \href{https://bae-9xf.pages.dev/composite/00190}{190} distinguishes between occurrences of `for'. \href{https://bae-9xf.pages.dev/composite/03021}{3021} distinguishes `in' tokens by whether the next token is a time or location.

Together, our quantitative and qualitative results point toward an abundance of non-linear structure in the language models we studied. Dominant concepts are often approximately linear, but their surrounding geometry is not, and one-dimensional decompositions discard it. This explains why composite autoencoders achieve better reconstructions and positions them as a promising tool and lens for extracting and viewing manifolds.

\subsection{Extracted subspaces are stable but individual latents are not}
\label{sub:similarity}

How consistent are autoencoders across runs? Existing answers rely on partial evidence: a subset of
encoder directions align across runs; reconstructions correlate. Bilinear autoencoders
allow an exact answer from weights alone. Each autoencoder is summarised by the kernel matrix
$\wei{K} = \wei{\mathcal{W}}\wei{\mathcal{W}}^\top$, and two such matrices admit either
a global Frobenius comparison (\autoref{eq:frobenius}) or a per-latent comparison via
Hungarian matching (\autoref{eq:hungarian}), after \citet{paulo2025sparseautoencoderstraineddata}.

\begin{equation}
\mathrm{Sim}_{F}(\wei{K}, \wei{K}') \;=\; 1 - \frac{\|\wei{K} - \wei{K}'\|_F^2}{\|\wei{K}\|_F^2 + \|\wei{K}'\|_F^2} \;=\; \frac{2 \Tr(\wei{K}^\top \wei{K}')}{\|\wei{K}\|^2_F + \|\wei{K}'\|^2_F}
\label{eq:frobenius}
\end{equation}

\begin{equation}
\mathrm{Sim}_{H}(\wei{\mathcal{W}}, \wei{\mathcal{W}}') \;=\; \frac{\|\wei{\mathcal{W}}^\top P \wei{\mathcal{W}}'\|_F}{\|\wei{\mathcal{W}'}^\top \wei{\mathcal{W}}'\|_F}
\quad \text{where} \quad
P = \argmax_{P \in \mathcal{P}_k} \Tr(\wei{\mathcal{W}}^\top P \wei{\mathcal{W}}')
\label{eq:hungarian}
\end{equation}

Where $\mathcal{P}_k$ denotes permutation of $k$ latents. The two metrics tell different stories. Across bilinear autoencoders that differ only in their sparsity penalty, individual latents agree only 15--50\%, while the global
subspace they span agrees above 90\% for similar sparsity levels (\autoref{fig:similarity}).
The gap widens across architectures: individual latents agree even less, while global
similarity remains high.

Current interpretability literature often discusses `true' latents and canonical recovery
\citep{paulo2025sparseautoencoderstraineddata, leask2025sparseautoencoderscanonicalunits,
song2025positionmechanisticinterpretabilityprioritize}. If different autoencoders recover
the same structure through different decompositions, the question becomes less about which
decomposition is true and more about which is interpretable, at competitive reconstruction
\citep{gauderis_mechanistic_2026}.

\begin{figure}[H]
  \centering
  \includegraphics[width=0.8\textwidth]{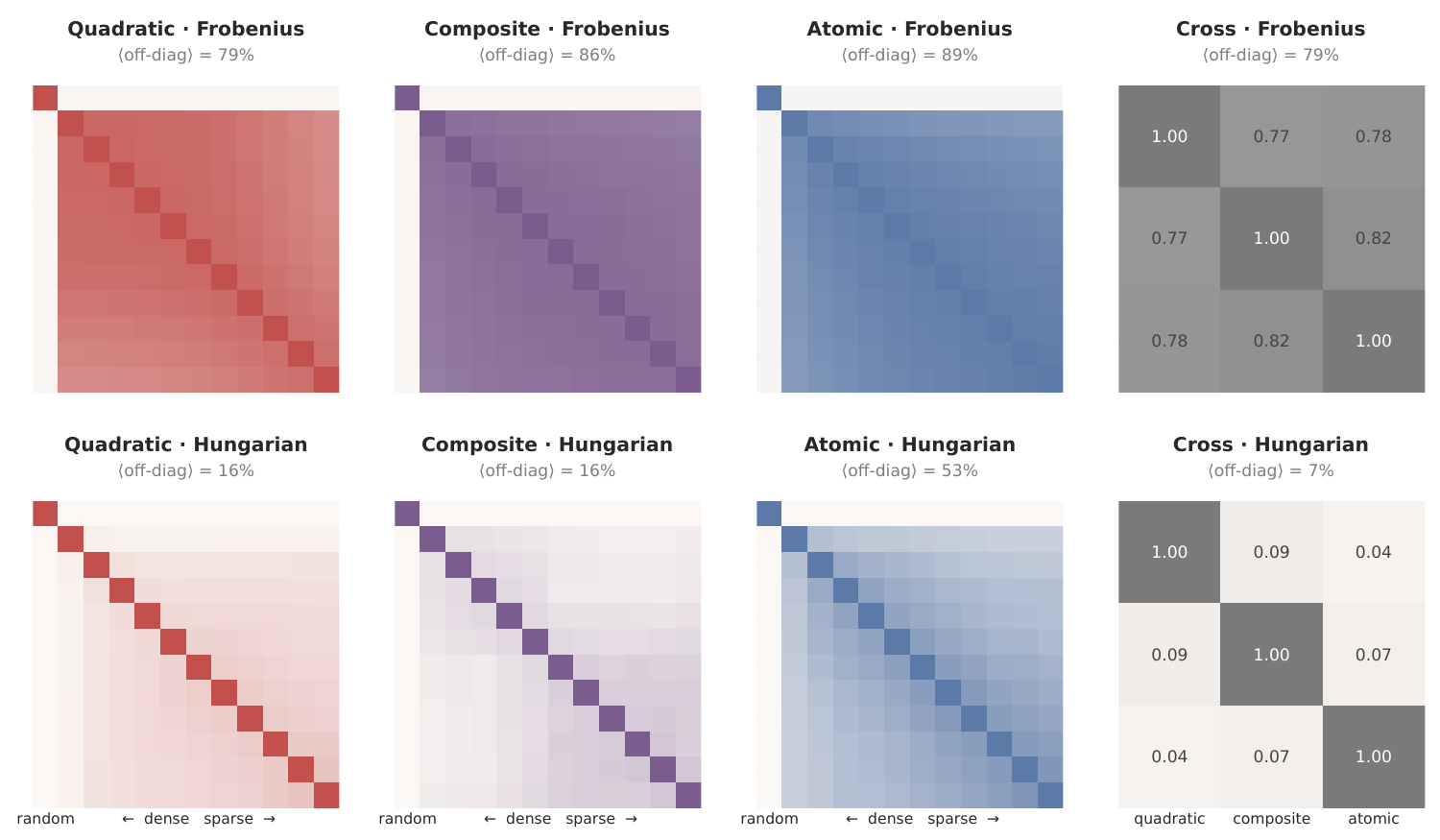}
  \caption{Frobenius (global) and Hungarian (per-latent) similarity between bilinear autoencoders with density penalties $\alpha \in [0, 1]$. Autoencoders with varying hyperparameters reconstruct highly similar subspaces, yet their latent decompositions vary significantly, especially for the composite latents. The first entry corresponds to a randomly initialised autoencoder; the last column corresponds to similarity between autoencoder kinds at $\alpha = 0.3$.}
  \label{fig:similarity}
\end{figure}

\section{Conclusion} \label{sec:conclusion}
\paragraph{Summary.}
Bilinear autoencoders decompose representations into polynomial latents that admit geometric analysis from weights alone. We argue that polynomials are an appealing non-linear primitive for interpretability: their composition stays inside the same algebraic class, and their geometry is read directly from their weight tensors. Section 3 formalises this through three nested hypotheses---atomic, composite, and quadratic---each a constraint on the mixing matrix $\wei{C}$ that trades expressivity for interpretability. The composite regime, a global topK on $\wei{C}$, recovers low-dimensional manifolds automatically, without the post-hoc co-activation analysis other methods rely on. Our quantitative and qualitative analysis reveals that higher-rank geometries are prevalent in language model activations, that our autoencoders recover them in an interpretable form, and that the recovered input subspace is stable despite disagreement in the individual atoms.

\paragraph{Future work.} This paper is intentionally scoped toward geometries and its interpretable extraction. Hence, the most fruitful future work is to causally ablate these latents. While manifold-based steering is inherently more involved than its linear counterpart, our autoencoder might make steering more reliable. By capturing the precise geometric structure through the weights, we can mathematically ensure we do not steer off-manifold and to reduce unintended effects of our edits. Next, the structural priors in this paper are a proof-of-concept; we did not sweep across priors to systematically assess which gave the best interpretative trade-off. We are excited about future work that builds atop polynomial latents and to characterise which kinds various models prefer. Finally, probing on quadratic latents is a natural next step given our results showing they capture concepts more naturally. 

\paragraph{Impact.} Bilinear autoencoders move beyond the linear representation hypothesis that dominates current interpretability, while remaining algebraically tractable. Because their computation and complexity admit a closed form, their description length is well-defined, which is rarely the case for nonlinear operations \citep{ayonrinde2024interpretabilitycompressionreconsideringsae}. We believe this shifts the focus of interpretability from the search for isolated interpretable bases towards understandable compositional structures and interactions \citep{gauderis_mechanistic_2026}.

\begin{ack}

\end{ack}

\bibliographystyle{plainnat}
\bibliography{refs}








\appendix

\section{Notation}
\label{app:notation}

We use colour to distinguish two roles that symbols play throughout
this paper. Red denotes \wei{weight-based} quantities:
parameters fixed after training, analysable without any input.
Blue denotes \act{activation-based} quantities: values
that depend on an input $\act{\mathbf{x}}$ and change with
each forward pass.

\paragraph{Dimensions.}
\begin{itemize}
  \item $d$ --- model (input) dimension.
  \item $h$ --- hidden dimension inside the bilinear encoder, indexed by $j$.
  \item $k$ --- number of latents (dictionary size), indexed by $i$.
  \item $n$ --- effective batch size (samples $\times$ sequence length).
\end{itemize}

\paragraph{Activations.}
\begin{itemize}
  \item $\act{\mathbf{x}} \in \mathbb{R}^d$ --- input activation vector.
  \item $\act{X} = \act{\mathbf{x}}\act{\mathbf{x}}^\top \in
        \mathbb{R}^{d \times d}$ --- lifted product-space representation;
        rank-1 positive-semidefinite.
  \item $\act{\hat{X}} \in \mathbb{R}^{d \times d}$ --- reconstructed
        product-space representation.
  \item $\act{\mathbf{z}} \in \mathbb{R}^k$ --- latent activation vector;
        the sparse code produced by the encoder.
  \item $\act{z}_i = \act{\mathbf{x}}^\top \wei{W}_i \act{\mathbf{x}}
        = \langle \wei{W}_i,\, \act{X} \rangle_F$ --- scalar activation
        of the $i$-th latent.
\end{itemize}

\paragraph{Encoder/Decoder weights.}
\begin{itemize}
  \item $\wei{L}, \wei{R} \in \mathbb{R}^{h \times d}$ --- left and right
        encoder factor matrices; $j$-th rows $\wei{\mathbf{l}}_j,
        \wei{\mathbf{r}}_j \in \mathbb{R}^d$.
  \item $\wei{C} \in \mathbb{R}^{k \times h}$ --- composition matrix;
        $i$-th row $\wei{\mathbf{c}}_i$ mixes hidden units into latent $i$;
        entry $\wei{C}_{ij}$ is the static weight of hidden unit $j$ in
        latent $i$.
\item $\wei{W}_i = \sum_{j \in \mathrm{supp}(\wei{\mathbf{c}}_i)} \wei{C}_{ij}\,
        \wei{\mathbf{l}}_j\wei{\mathbf{r}}_j^\top
        \in \mathbb{R}^{d \times d}_{\mathrm{sym}}$ --- symmetric bilinear
        form of the $i$-th latent. Since $\act{\mathbf{x}}^\top A
        \act{\mathbf{x}} = \act{\mathbf{x}}^\top \tfrac{A+A^\top}{2}
        \act{\mathbf{x}}$, only the symmetric part of any bilinear form
        contributes, so $\wei{W}_i$ is taken symmetric without loss of
        generality.
  \item $\wei{\mathcal{W}} \in \mathbb{R}^{k \times (d \times d)}$ ---
        full encoder tensor; slice $\wei{\mathcal{W}}_i = \wei{W}_i$.
  \item $\wei{K} = \wei{\mathcal{W}}\wei{\mathcal{W}}^\top
        = \wei{C}^\top(\wei{L}\wei{L}^\top \odot \wei{R}\wei{R}^\top)
        \wei{C} \in \mathbb{R}^{k \times k}$ --- kernel (Gram) matrix;
        allows evaluating the reconstruction loss without materialising
        the $d^2$-dimensional product space.
\end{itemize}

\paragraph{Dictionary atoms (representation hypotheses).}
\begin{itemize}
  \item $\wei{W}_i \in \mathbb{R}^{d \times d}_{\mathrm{sym}}$ ---
        $i$-th dictionary atom (quadratic hypothesis); the same symbol
        as the encoder bilinear form.
  \item $\wei{\mathbf{w}}_i \in \mathbb{R}^d$ --- $i$-th atomic
        dictionary direction; under the atomic hypothesis
        $\wei{W}_i = \wei{\mathbf{w}}_i\wei{\mathbf{w}}_i^\top$.
  \item Under the composite hypothesis, $\wei{W}_i = \sum_{j \in
        \mathrm{supp}(\wei{\mathbf{c}}_i)} \wei{C}_{ij}\,
        \wei{\mathbf{w}}_j\wei{\mathbf{w}}_j^\top$; no additional
        letter is needed beyond $\wei{C}_{ij}$.
\end{itemize}

\paragraph{Spectral quantities (weight-based).}
\begin{itemize}
  \item $\wei{W}_i = \wei{U}_i \wei{\Lambda}_i \wei{U}_i^\top$ ---
        eigendecomposition of the $i$-th latent's bilinear form.
  \item $\wei{\lambda}_j^{(i)},\, \wei{\mathbf{v}}_j^{(i)}$ ---
        $j$-th eigenvalue and eigenvector of $\wei{W}_i$; computable
        from weights alone, no forward pass required.
\end{itemize}

\paragraph{Similarity metrics.}
\begin{itemize}
  \item $\mathrm{Sim}_F(\wei{\mathcal{W}}, \wei{\mathcal{W}}') =
        2\,\mathrm{Tr}(\wei{K}^\top \wei{K}') /
        (\|\wei{K}\|_F^2 + \|\wei{K}'\|_F^2)$ ---
        Frobenius (global) similarity between two autoencoders.
  \item $\mathrm{Sim}_H(\wei{\mathcal{W}}, \wei{\mathcal{W}}')$ ---
        Hungarian (per-latent) similarity after optimal assignment.
\end{itemize}

\section{Experimental setup} \label{app:setup}

We use the same hyperparameters across all discussed autoencoders. Hyperparameters follow
standard practice for sparse autoencoders; results are robust across the variations we
tested (\autoref{app:models}).

\paragraph{Data.} In \autoref{sec:empirics}, we exclusively used
\href{https://huggingface.co/Qwen/Qwen3.5-0.8B}{\texttt{Qwen3.5-0.8B-Base}}
\citep{qwen35blog}---a small but competent model---with input samples from the
\href{https://huggingface.co/datasets/HuggingFaceFW/fineweb}{FineWeb dataset}
\citep{penedo2024the}. This lets us sweep over reasonably sized autoencoders and fully
train on a tight compute budget. All autoencoders train in about 15 minutes on an RTX
4080. We replicate these results on other open source models, discussed in
\autoref{app:models}.

\paragraph{Optimiser.} We use the Muon optimiser \citep{jordan2024muon}, which we have
found to significantly outperform Adam, sometimes up to $5\times$ faster convergence. We
use a constant learning rate schedule. We use a linear warmup for $\alpha$ to prioritise
reconstruction over sparsity near the start. We use topK annealing to sparsify
$\wei{C}$, leaving only $0.1\%$ of entries active: we decay to the target sparsity over the first half of training and freeze the sparsity mask for the last $20\%$, which stabilises convergence.

\paragraph{Architecture.} We have found orthogonal initialisation
to perform better than Xavier initialisation. Orthogonal initialisation favours sparser
solutions at a minute reconstruction cost. We use $8,192$ latents and $16,384$ hidden dimensions across models, and $\alpha=0.3$ by default. Weight-tying the encoder and decoder does not impact reconstruction and hence do so. 

\begin{figure}[H]
  \centering
  \begin{minipage}{0.31\textwidth}
    \centering
    \begin{tabular}{@{}ll@{}}
      \textbf{Batch size}     & $32$ \\
      \textbf{Sequence length} & $256$ \\
      \textbf{Steps}          & $2^{11}$ \\
      \textbf{Total tokens}   & $\approx 8\mathrm{M}$ \\
      \textbf{Shuffling}      & No \\
    \end{tabular}
  \end{minipage}\hfill
  \begin{minipage}{0.31\textwidth}
    \centering
    \begin{tabular}{@{}ll@{}}
      \textbf{Optimiser}      & Muon \\
      \textbf{Momentum}       & $0.95$ \\
      \textbf{Learning rate}  & $0.03$ \\
      \textbf{Schedule}       & constant \\
      \textbf{$\alpha$ warmup} & $256$ steps \\
    \end{tabular}
  \end{minipage}\hfill
  \begin{minipage}{0.34\textwidth}
    \centering
    \begin{tabular}{@{}ll@{}}  
      \textbf{$d$ (model)}     & $1024$ \\
      \textbf{$k$ (latents)}   & $8 \times 1024$ \\
      \textbf{$h$ (hidden)}    & $16 \times 1024$ \\
      \textbf{Default $K$}     & $0.1\%$ \\
      \textbf{Default $\alpha$} & $0.3$ \\
    \end{tabular}
  \end{minipage}
  \caption{Hyperparameters related to data, the optimiser, and the architecture,
  respectively.}
  \label{fig:hyperparameters}
\end{figure}

Finally, our code is publicly available at \codelink and extensively uses PyTorch \citep{paszke2019pytorchimperativestylehighperformance}.

\section{Pseudocode} \label{app:code}

We provide distilled pseudo code for the initialisation and forward pass of the `atomic' autoencoder. The intuition for why this works is provided in \autoref{app:efficiency}.

\begin{minipage}{\linewidth}
\begin{lstlisting}[style=mintedclean]
def __init__(self, config):
    self.left = nn.init.orthogonal(config.d_features, config.d_model)
    self.right = nn.init.orthogonal(config.d_features, config.d_model)

def forward(self, acts: Tensor['batch dims'], alpha: float):
    # Compute the latents
    latents = linear(acts, self.left) * linear(acts, self.right)

    # Compute the density and regularisation term
    density = hoyer(latents, dim=0).mean()

    # Compute constituent parts of the reconstruction error
    kernel = (self.left @ self.left.T) * (self.right @ self.right.T)
    recons = einsum(latents, latents, kernel, "... h1, ... h2, h1 h2 -> ...")
    cross = latents.square().sum(-1)

    # Compute the error and loss
    target = acts.square().sum(-1, keepdim=True)
    error = ((recons - 2 * cross + target) / target).mean()
    loss = error + alpha * density
    return loss, dict(error=error, density=density)
\end{lstlisting}
\end{minipage}

\section{Kernel trick for efficient training}
\label{app:efficiency}

Bilinear autoencoders differ from ordinary ones because they reconstruct the product
space. While the inputs are naturally factored as $\act{X} = \act{\mathbf{x}}\act{\mathbf{x}}^\top$,
the outputs are not: $\act{\hat{X}} \neq \act{\hat{\mathbf{x}}}\act{\hat{\mathbf{x}}}^\top$.
Hence, naively computing the reconstruction loss requires materialising the quadratic
space, which is unmanageable at scale. However, this can be avoided by expanding the
numerator (using $\act{\mathbf{z}} := \wei{\mathcal{W}}\act{X}$):
\begin{align}
\label{eq:loss}
\mathrm{NMSE}(\act{\hat{X}}, \act{X}) \;=\;
\frac{\|\act{\hat{X}} - \act{X}\|^2_F}{\|\act{X}\|^2_F} =
\frac{\act{\mathbf{z}}^\top \wei{K} \act{\mathbf{z}}
      - 2 \|\act{\mathbf{z}}\|^2
      + \|\act{X}\|^2_F}{\|\act{X}\|^2_F}
\end{align}

\begin{figure}[H]
  \centering
  \includegraphics[width=0.9\textwidth]{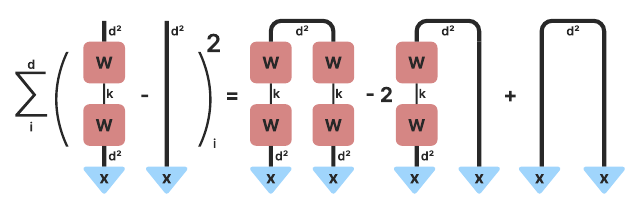}
  \caption{Diagrammatic formulation of \autoref{eq:loss}. Lines indicate tensor
  contractions over that index, while unconnected lines represent unmodified open
  indices. Thick lines represent a quadratic space.}
  \label{fig:loss}
\end{figure}

The reconstruction $\act{\hat{X}}$ never appears. All three terms are inner products of
$\act{\mathbf{x}}$ and $\act{\mathbf{z}}$, and the first is weighted by
$\wei{K}$, which is efficiently computable via
\autoref{eq:kernel}. This is reminiscent of the kernel trick to simulate an intractable
polynomial feature space through inner products, except that here $\wei{K}$ is learned by
the autoencoder. Due to the factorisation of $\wei{\mathcal{W}}$, we can compute
$\wei{K}$ without materialising any higher-order tensors:
\begin{equation}
\wei{K} \;=\; \wei{\mathcal{W}}\wei{\mathcal{W}}^\top
           \;=\; \wei{C}^\top(\wei{L}\wei{L}^\top \odot \wei{R}\wei{R}^\top)\wei{C}
\label{eq:kernel}
\end{equation}

\begin{figure}[H]
  \centering
  \includegraphics[width=0.6\textwidth]{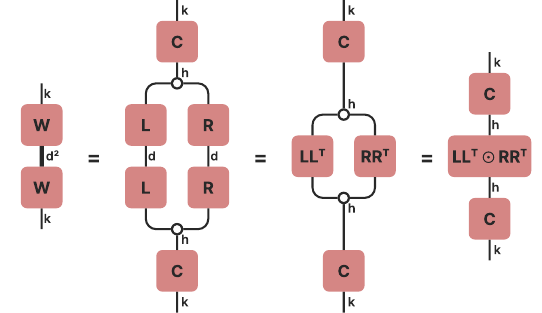}
  \caption{Diagrammatic equation for computing the kernel
  $\wei{K} = \wei{\mathcal{W}}\wei{\mathcal{W}}^\top$ efficiently using a sensible
  contraction order. The kernel matrix of a bilinear autoencoder can be evaluated by
  recasting it as an element-wise multiplication of two matrices.}
  \label{fig:kernel}
\end{figure}

Consequently, bilinear autoencoders scale quadratically in compute with latent count $k$ while their memory usage is identical to that of an ordinary autoencoder. While this trick reduces
computational complexity, the kernel matrix $\wei{K} \in \mathbb{R}^{k \times k}$ can
still be a memory bottleneck. However, it can be evaluated by splitting the $k \times k$
space into $p$ blocks $s_1 \oplus s_2 \oplus \dots \oplus s_p$, which avoids full
materialisation:
\[
\act{\mathbf{z}}^\top \wei{K} \act{\mathbf{z}} =
\sum_{s_a, s_b} \act{\mathbf{z}}_{s_a}\, \wei{K}_{s_a s_b}\, \act{\mathbf{z}}_{s_b}
\]
This can be further improved since only the symmetric part of $\wei{K}$ contributes;
sampling only the lower or upper triangular part speeds up computation by a factor of $2$.

\section{Proof of \autoref{thm:receptive-field}}
\label{app:proof}

\begin{proof}[Proof of \autoref{thm:receptive-field}]
Since $\act{\mathbf{x}} \sim \mathrm{Unif}(S^{d-1})$ and $\wei{\mathbf{w}} \in S^{d-1}$,
the Poincar\'{e}--Maxwell--Borel lemma gives
$\sqrt{d}\,(\wei{\mathbf{w}}^\top\act{\mathbf{x}}) \xrightarrow{\mathcal{D}} Z \sim \mathcal{N}(0,1)$
as $d\to\infty$. Together with concentration of $\|\mathbf{g}\|$ around $\sqrt{d}$
for $\mathbf{g}\sim\mathcal{N}(0,I_d)$, this yields, for fixed $\tau$,
\[
  \mathbb{P}(\wei{\mathbf{w}}^\top\act{\mathbf{x}} > \tau)
        \;\sim\; \mathbb{P}(Z > \tau\sqrt{d}),
  \qquad
  \mathbb{P}\!\left((\wei{\mathbf{w}}^\top\act{\mathbf{x}})^2 > \tau\right)
        \;\sim\; 2\,\mathbb{P}(Z > \sqrt{\tau d}).
\]
Applying the standard Gaussian tail asymptotic via Mills'
ratio, $\mathbb{P}(Z > a) = \Theta(a^{-1}e^{-a^2/2})$ for $a > 0$:
\[
  \frac{\mathbb{P}\!\left((\wei{\mathbf{w}}^\top\act{\mathbf{x}})^2 > \tau\right)}
       {\mathbb{P}(\wei{\mathbf{w}}^\top\act{\mathbf{x}} > \tau)}
  = \frac{\Theta\!\left((\tau d)^{-1/2}\,e^{-\tau d/2}\right)}
         {\Theta\!\left((\tau^2 d)^{-1/2}\,e^{-\tau^2 d/2}\right)}
  = \Theta\!\left(\sqrt{\tau}\;e^{-d\tau(1-\tau)/2}\right)
  = \Theta\!\left(e^{-d\tau(1-\tau)/2}\right),
\]
where the last equality absorbs the fixed constant $\sqrt{\tau}$ into the $\Theta$.
Since $\tau\in(0,1)$, the exponent is strictly negative, confirming exponential
decay in $d$.
\end{proof}

\section{Product space error correction}
\label{app:correction}

This section discusses the discrepancy between reconstructing $\act{\mathbf{x}}$ and
$\act{\mathbf{x}}\act{\mathbf{x}}^\top$. This product-space correction admits an exact
closed form. Taking $\|\act{\mathbf{x}}\|^2 = 1$ due to scale-invariance
(\autoref{eq:loss}) and substituting
$\act{\mathbf{x}}^\top \act{\hat{\mathbf{x}}} =
\tfrac{1}{2}(1 + \|\act{\hat{\mathbf{x}}}\|^2 - s)$
with $s = \|\act{\mathbf{x}} - \act{\hat{\mathbf{x}}}\|^2$,
\begin{align}
S &= \|\act{\mathbf{x}}\act{\mathbf{x}}^\top
     - \act{\hat{\mathbf{x}}}\act{\hat{\mathbf{x}}}^\top\|^2_F \nonumber \\
  &= 1 + \|\act{\hat{\mathbf{x}}}\|^4
       - 2\,(\act{\mathbf{x}}^\top \act{\hat{\mathbf{x}}})^2 \nonumber \\
  &= \tfrac{1}{2}(1 - \|\act{\hat{\mathbf{x}}}\|^2)^2
   + s(1 + \|\act{\hat{\mathbf{x}}}\|^2)
   - \tfrac{1}{2}s^2.
\label{eq:quadratic-recons}
\end{align}
Under the mild assumption $\|\act{\hat{\mathbf{x}}}\|^2 \approx 1$, which any sensible
reconstruction should satisfy, the first term vanishes and the third is negligible for
small $s$, leaving $S \approx 2s$. The product-space loss is therefore approximately
twice the input-space loss. We use the exact correction in \autoref{fig:layers}, which
consequently shows that bilinear autoencoders achieve better reconstruction than the
baseline across most model layers and variants in the product space. This suggests that
bilinear autoencoders can leverage decodable structure in interactions between input
elements to improve reconstruction.

\section{Results on additional models} \label{app:models}

We extend our empirical results from \autoref{sub:reconstruction} on two other
well-known open-source models: Gemma-3-1B \citep{gemmateam2025gemma3technicalreport}
and Llama-3.2-1B \citep{grattafiori2024llama3herdmodels}. These are in line with
expectations from Qwen3.5-0.8B.

\begin{figure}[H]
  \centering
  \includegraphics[width=\textwidth]{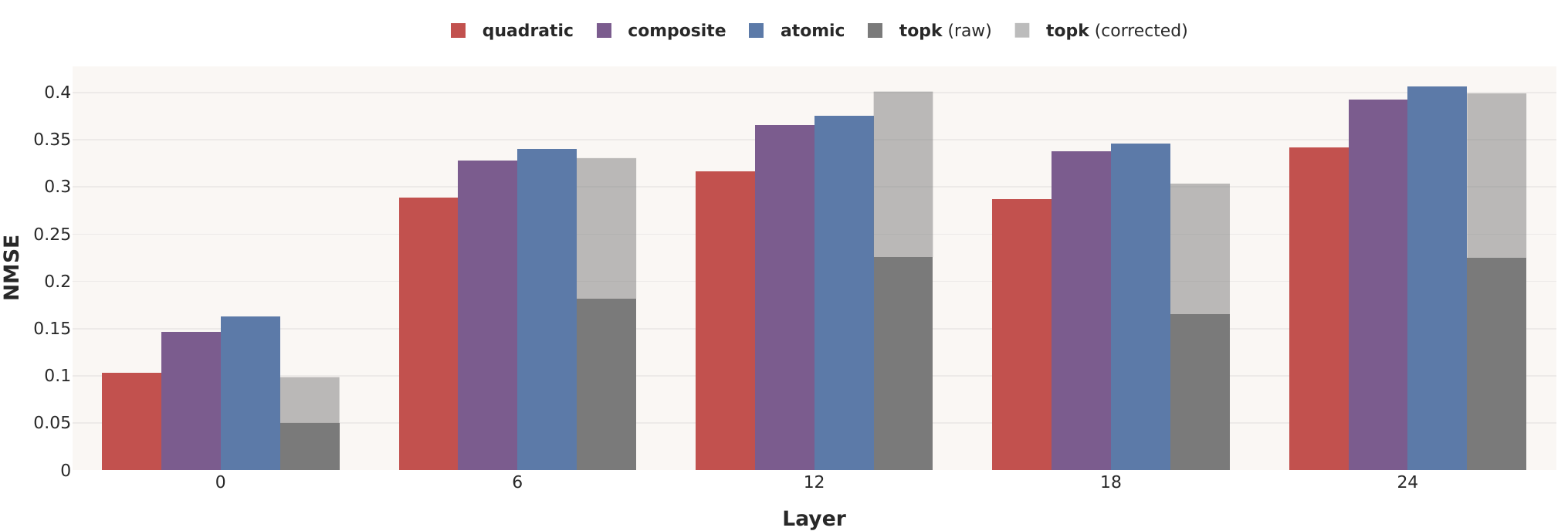}
  \caption{Sweep over selected layers on Gemma-3-1B using the hyperparameters discussed
  in \autoref{app:setup}. To clarify, $k = 16\mathrm{k}$ (i.e.\ not scaled with the
  slightly increased model dimension of $d = 1152$). The results are in line with
  \autoref{sub:reconstruction}, likely because the model width and depth resemble
  Qwen3.5-0.8B.}
  \label{fig:layers-gemma}
\end{figure}

\begin{figure}[H]
  \centering
  \includegraphics[width=\textwidth]{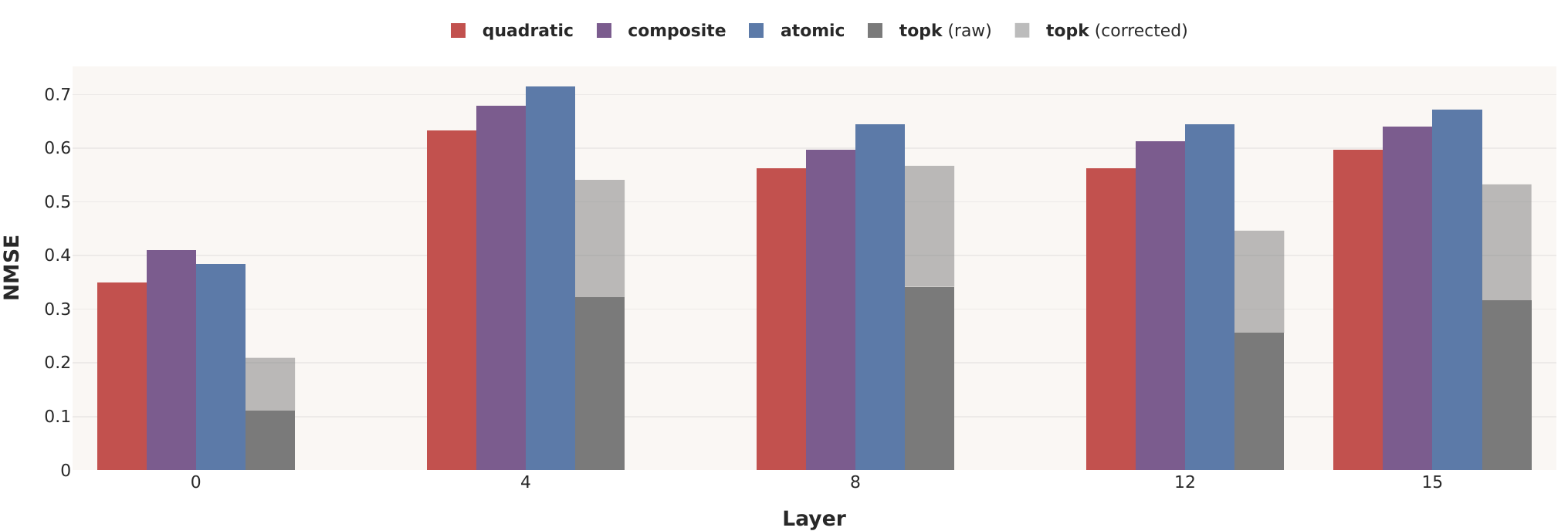}
  \caption{Sweep over selected layers on Llama-3.2-1B using the hyperparameters
  discussed in \autoref{app:setup}. Reconstruction is much worse for both the topK
  baseline and our models, which is expected since $k$ and $h$ are kept constant despite
  $d = 2048$ being twice as large; this means the quadratic space is $4\times$ as large.}
  \label{fig:layers-llama}
\end{figure}

\section{Interactive viewer} \label{app:viwer}

To make the extracted geometries easy to inspect, we publish an interactive viewer at \url{https://bae-9xf.pages.dev} covering all 8{,}192 latents from a composite autoencoder trained on layer 18 of Qwen3.5-0.8B with ${\approx}64\text{M}$ tokens, $8\times$ longer than for \autoref{fig:layers}.

The landing page lays out the corpus as a scatter over four of the summary statistics defined below, with an alternative UMAP view. The 20 highest-importance latents are highlighted in a sidebar, each carrying an unevaluated 2--4 word label produced by a single Claude pass over its top-activating contexts.

\begin{figure}[H]
  \centering
  \includegraphics[width=0.8\textwidth]{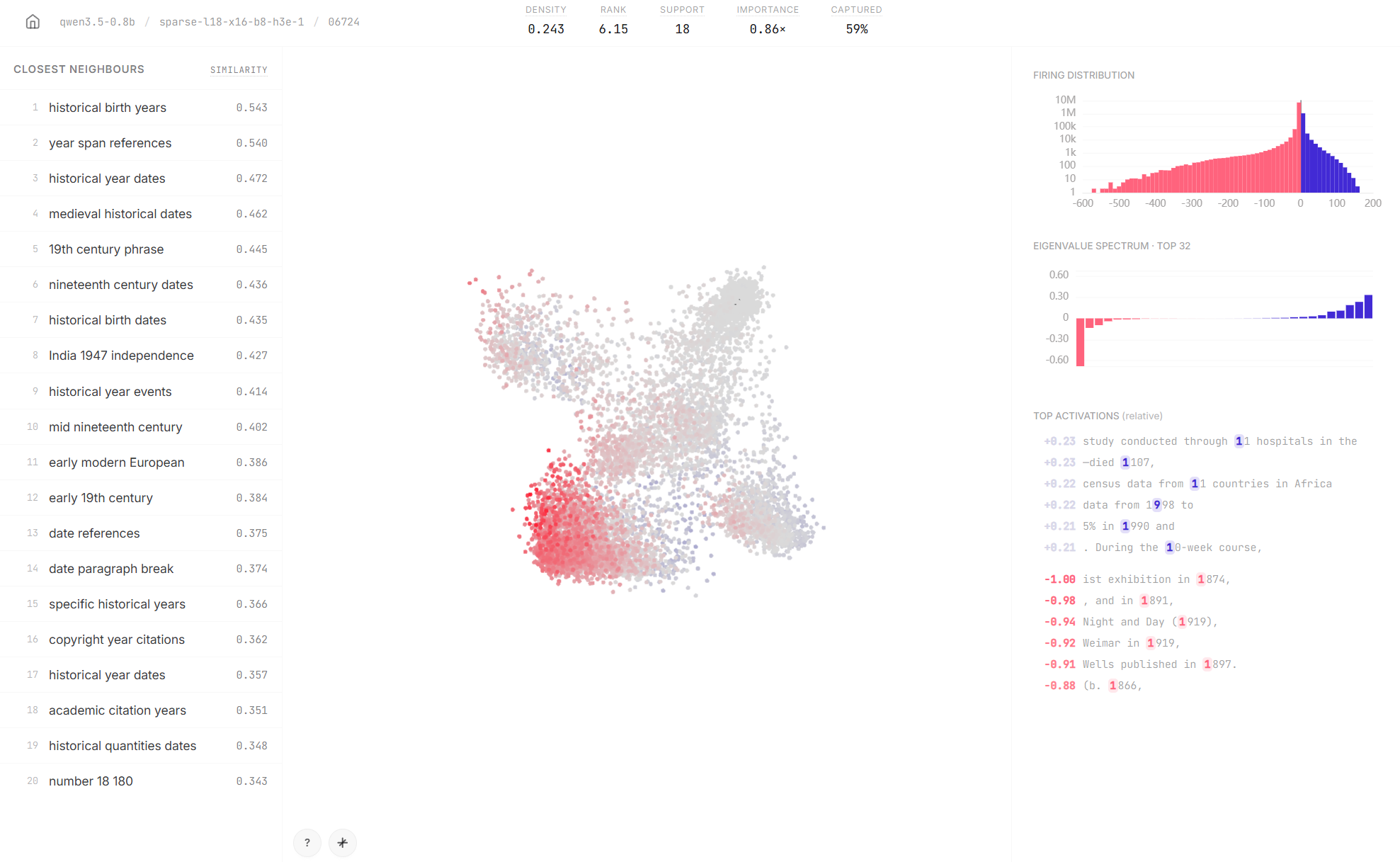}
  \caption{Screenshot of the interactive viewer user interface for a particular latent: \url{https://bae-9xf.pages.dev/composite/06724}.}
  \label{fig:viewer}
\end{figure}

Each composite has a dedicated page anchored by a 3D scatter of token contexts, projected onto the top-3 eigenvectors of $\wei{W}_i = \wei{U}_i \wei{\Lambda}_i \wei{U}_i^\top$ (selected by $|\wei{\lambda}|$). Eigendirections optionally overlay the scatter as bidirectional axes (purple for positive $\wei{\lambda}$, pink for negative; bidirectional because the bilinear form depends only on $(\wei{\mathbf{v}}_j^{(i)\top}\act{\mathbf{x}})^2$). Points are drawn from ${\sim}8$M tokens by weighted reservoir sampling $w \propto (\|\act{\mathbf{z}}\| + \varepsilon)^4$, over-representing samples far from the origin so the manifold isn't drowned by the dense central cluster. Surrounding panels show the eigenvalue spectrum (with the three rendered axes labelled $X$, $Y$, $Z$), the 6 highest-activating $\pm 4$-token contexts per firing sign, and the 20 closest neighbours ranked by subspace overlap $\|\wei{U}_i^\top \wei{U}_j\|_F^2 / \sqrt{r_i r_j}$.

The header reports five summary statistics. Density measures how concentrated the firing distribution is. Effective rank gauges how many eigenvalues carry meaningful weight. Support counts the encoder atoms feeding into $\wei{W}_i$. Importance reports the total spectral energy, normalised so the average composite reads $1\times$. Captured is the spectral share of the three rendered eigenvectors and signals how faithfully the 3D scatter represents the full manifold.
\begin{align*}
\text{density} \;&=\; \mathrm{Hoyer}(\act{\mathbf{z}}_i), &
\text{importance} \;&=\; \textstyle\sum_j (\wei{\lambda}_j^{(i)})^2, \\[4pt]
\text{effective rank} \;&=\; \frac{(\sum_j |\wei{\lambda}_j^{(i)}|)^2}{\sum_j (\wei{\lambda}_j^{(i)})^2}, &
\text{captured} \;&=\; \frac{|\wei{\lambda}_1^{(i)}| + |\wei{\lambda}_2^{(i)}| + |\wei{\lambda}_3^{(i)}|}{\sum_j |\wei{\lambda}_j^{(i)}|}.
\end{align*}

\section{Random viewer examples} \label{app:examples}

To avoid selection bias, we show 6 random latents and their geometry.

\begin{figure}[H]
  \centering
  \includegraphics[width=\textwidth]{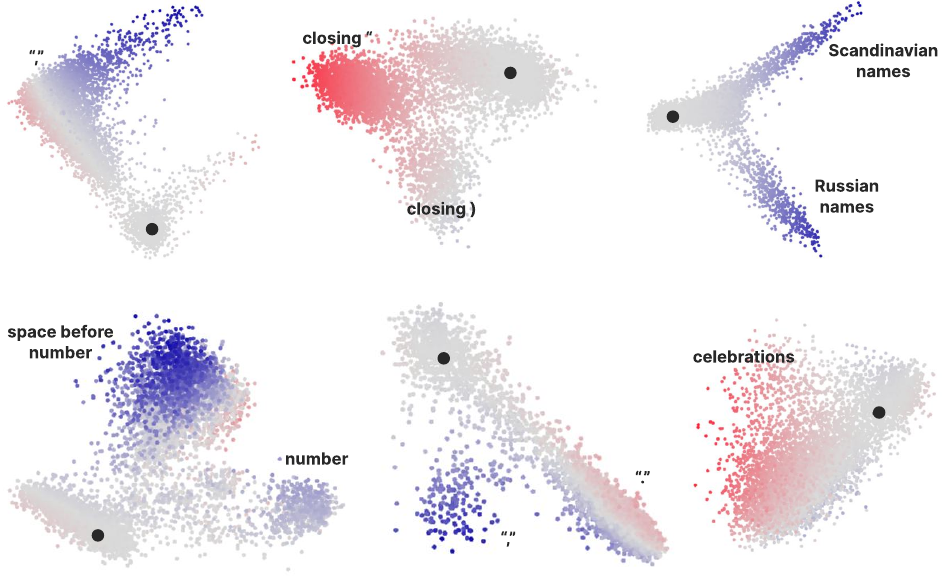}
  \caption{The first few manifolds by index from our interactive viewer (indices 1 and 4 were excluded as they are dense and uninteresting, about 10\% of latents are dense). This is a representative sample for how many features are primarily linear (top-left and top-right) versus more cluster-based (the others). }
  \label{fig:random-manifolds}
\end{figure}



\end{document}